\documentclass{article}
\usepackage{spconf,amsmath,graphicx, amssymb, amsthm, dsfont, amsfonts, subcaption}
\newcommand{\bS}[1]{{\boldsymbol{#1}}}

\long\def\comment#1{}

\newfont{\bbb}{msbm10 scaled 700}

\newfont{\bb}{msbm10 scaled 1000}

\newtheorem{theorem}{Theorem}[section]
\newtheorem{corollary}{Corollary}[theorem]
\newtheorem{lemma}[theorem]{Lemma}

\usepackage[hidelinks]{hyperref}

\title{CHANNEL REDUNDANCY AND OVERLAP IN CONVOLUTIONAL NEURAL NETWORKS WITH CHANNEL-WISE NNK GRAPHS}
\name{David Bonet$^{\star}$ \qquad Antonio Ortega$^{\dagger}$ \qquad Javier Ruiz-Hidalgo$^{\star}$ \qquad Sarath Shekkizhar$^{\dagger}$}

\address{$^{\star}$ Universitat Politècnica de Catalunya, Barcelona, Spain \\ 
         $^{\dagger}$ University of Southern California, Los Angeles, USA}

\begin{document}
\ninept
\maketitle
\begin{abstract}


Feature spaces in the deep layers of convolutional neural networks (CNNs) are often very high-dimensional and difficult to interpret. However, convolutional layers consist of multiple channels that are activated by different types of inputs, which suggests that more insights may be gained by studying the channels and how they relate to each other. In this paper, we first analyze theoretically channel-wise non-negative kernel (CW-NNK) regression graphs, which allow us to quantify the overlap between channels and, indirectly, the intrinsic dimension of the data representation manifold. We find that redundancy between channels is significant and varies with the layer depth and the level of regularization during training. Additionally, we observe that there is a correlation between channel overlap in the last convolutional layer and generalization performance. Our experimental results demonstrate that these techniques can lead to a better understanding of deep representations. 

\end{abstract}
\section{Introduction}
\label{sec:intro}

Convolutional neural networks (CNNs) are among the most successful models for supervised classification \cite{tan2019efficientnet, radosavovic2020designing}.
In general, CNNs sequentially process inputs through blocks of convolutional layers followed by non-linearity, normalization and pooling layers, and fully connected layers with  softmax activation for classification. Each convolutional layer consists of multiple convolutional filters that are applied in parallel to a common input. Following each layer we encounter very high dimensional data representations, formed by the aggregation of multiple \textit{channels}, where each channel is the part of the output that corresponds to a single convolutional filter.
It is generally accepted that initial layers encode basic features \cite{lecun2015deep}, e.g., edges and corners, while the features learned in the deeper layers are harder to interpret. Hence, the gap in the application of deep learning models in more sensitive practices such as medicine or defense remains significant and our understanding of how these models generalize to unseen data points is still limited \cite{zhang2021understanding}.


In this paper, we use graphs to study the geometric properties of the training data representations and achieve a better understanding of CNNs.  
Interest in the geometric properties of high dimensional datasets has led to various attempts to define the \textit{intrinsic dimension} (ID) of the data.
While multiple definitions are possible \cite{campadelli2015intrinsic}, the general idea is that, for a given type of data representation, ID is related to the number of parameters needed to describe a dataset accurately \cite{bennett1969intrinsic}. 
Measuring ID helps to understand how neural networks learn to transform data effectively \cite{ansuini2019intrinsic}, even when there is significant overparametrization. Further, specific local ID characteristics can be associated with desirable properties, such as adversarial robustness \cite{ma2018characterizing} or  generalization for noisy labels \cite{ma2018dimensionality}.
Recently, \cite{ansuini2019intrinsic, recanatesi2019dimensionality} study ID for several CNN architectures with different training techniques and show
that the IDs of deep representations are notably smaller than the nominal dimension of the corresponding features. More generally, these works observe that the ID increases in the initial layers and then decreases closer to the final classification layer. \cite{ansuini2019intrinsic, recanatesi2019dimensionality} attribute the initial increase in ID in the early layers 
to the fact that these layers perform low-level  pre-processing and feature extraction, many of which are task-independent. However, in the later layers, only task-relevant features are learned,  and the ability of the model to compress the dimensionality of data representations in these last layers is indicative of the model's generalization.

Given the large number of channels that comprise each layer of state-of-the-art CNNs, and the computation and memory they require, there has been a growing interest in understanding how these channels are related to each other \cite{shang2016understanding, wang2020orthogonal, kahatapitiya2021exploiting}. In particular, redundancy between channels has been studied based on various pairwise filter similarities, detecting pairs of filters with opposite phase \cite{shang2016understanding}, using guided back-propagation patterns \cite{wang2020orthogonal}, and analyzing filter correlation \cite{kahatapitiya2021exploiting}. Channel similarity estimates can be used to compress models by pruning channels that are similar in a layer \cite{he2017channel, jaderberg2014speeding} or by leveraging the structural redundancy of the channels in a layer \cite{wang2021convolutional}.
Orthogonality constraints in convolutional layers have also been proposed to de-correlate channels during training \cite{wang2020orthogonal, kahatapitiya2021exploiting, rodriguez2016regularizing}. 


In this work, we study channel redundancy directly using graphs constructed from training data, rather than relying on a more indirect metric, such as filter similarity  \cite{shang2016understanding, wang2020orthogonal, kahatapitiya2021exploiting}. 
We propose a novel way to estimate the ID of data representations from empirical data
based on the level of information overlap across channels.
To this aim, we construct and study channel-wise non-negative kernel (NNK) regression graphs \cite{shekkizhar2020graph} using data representations at the output of each convolutional layer.
NNK graphs take advantage of the local geometry to sparsify a $K$-nearest neighbor (KNN) graph. 
Further, NNK neighborhoods are stable for large enough $K$, and their size is indicative of the ID of the manifold the data belongs to \cite{shekkizhar2020graph}. 
This method has been shown to provide advantages for semi-supervised learning, image representation \cite{shekkizhar2020efficient}, and label interpolation and generalization estimation in neural networks \cite{shekkizhar2020deepnnk}. 
Graph properties have also been used to interpret deep neural network performance \cite{GriOrtGir20182}, latent space geometry \cite{LasBonHacGriTanOrt2020, LasGriOrt2021laplacian}, and to improve model robustness \cite{LasGriOrt2021representing}. However, none of these papers used channel-wise information or the NNK construction. 
We first proposed a channel-wise approach for NNK graph constructions in \cite{bonet2021channel}, where information from multiple channels was used to estimate generalization and perform early stopping without requiring a validation set while relying on a task performance metric (label interpolation).
In this paper, we 
analyze the conditions that guarantee that two nodes that are NNK graph neighbors in two channels, i.e., overlap, will also be neighbors in an NNK graph constructed by aggregating the two channels. These results allow us to quantify the level of similarity between data points (across multiple channels) without a combinatorial explosion in the number of graphs to be constructed.
From our experimental results, we find that channel redundancy in CNNs is significant and varies according to the layer depth and the level of training regularization. Our method for assessing the redundancy in CNNs can be leveraged to improve or create new channel pruning techniques to obtain more compact and efficient models. We also observe that CW-NNK overlap in the last convolutional layer is strongly correlated with generalization performance. Our measure of overlap can be combined with the recently proposed channel-wise early stopping method \cite{bonet2021channel} to improve generalization estimation and efficiency for large scale problems.

\section{Preliminaries}
\label{sec:pre}

\subsection{Notation}
\label{ssec:notation}

We denote scalars, vectors, and matrices using lowercase (e.g., $x$ and $\theta$), lowercase bold (e.g., $\bS{x}$ and $\bS{\theta}$), and uppercase bold (e.g., $\bS{K}$ and $\bS{\Phi}$) letters, respectively. Subscript $i$ denotes the $i^{\text{th}}$ data point associated with the feature vector $\bS{x}_i$.
We use a double subscript to index and denote a channel and related measures. For example, a feature vector $\bS{x}_i \in \mathbb{R}^D$ obtained as the concatenation of $C$ channel subvectors $\bS{x}_{i_c} \in \mathbb{R}^{D_c}$ where $\sum_{c=1}^{C} D_c = D$.
The indices of the KNN and NNK neighbors of a data point $\bS{x}_i$ in the aggregate space are denoted by sets $\mathcal{S}(\bS{x}_i)$ and $\mathcal{N}(\bS{x}_i)$, while $\mathcal{S}(\bS{x}_{i_c})$ and $\mathcal{N}(\bS{x}_{i_c})$ are the neighborhood sets for the channel subvector $\bS{x}_{i_c}$, 
respectively.

\subsection{Non-Negative Kernel (NNK) regression graphs}
\label{ssec:nnk}

Given a set of $N$ data points represented by feature vectors $\bS{x} \in \mathbb{R}^D$, a graph is constructed by 
connecting each data point (node) to similar data points, so that the weight of an edge  
between two nodes is based on the similarity of the data points, with the absence of an edge (a zero weight) 
denoting least similarity. 
Conventionally, one defines similarity between data points using positive definite kernels \cite{aronszajn1950theory} such as the Gaussian kernel with bandwidth $\sigma$: 
\begin{equation}
\label{eq:gaussian}
    k(\bS{x}_i,\bS{x}_j) = \text{exp} \left( -\frac{\|\bS{x}_i - \bS{x}_j\|^2}{2 \sigma^2} \right)  
\end{equation}
Unlike weighted KNN and $\epsilon$-neighborhood graphs \cite{chvatal1977aggregations} that are sensitive to the choice of hyperparameters $K/\epsilon$, NNK graphs \cite{shekkizhar2020graph} are suggested as a principled approach to graph construction based on a signal representation view. 


An advantage of NNK over KNN, which can be viewed as a thresholding-based representation, is its robustness to the choice of parameter $K$. 
While KNN is still used as an initialization, NNK performs a further optimization akin to orthogonal matching pursuits  \cite{tropp2007signal} in kernel space, resulting in a \emph{stable} representation (even as $K$ chosen for initialization grows, the number of NNK neighbors remains nearly constant) with the added advantage of having a \emph{geometric} interpretation based on the Kernel Ratio Interval (KRI) theorem \cite{shekkizhar2020graph}:
\begin{equation} \label{eq:kri}
    \bS{K}_{j,k} < \frac{\bS{K}_{i,j}}{\bS{K}_{i,k}} < \frac{1}{\bS{K}_{j,k}}
\end{equation}
where $\bS{K}_{i,j} = k(\bS{x}_i,\bS{x}_j)$.
In the case of the Gaussian kernel \eqref{eq:gaussian}, considering the edge $\theta_{ij}$ connecting node $i$ and node $j$, we can define a hyperplane normal to the edge direction. The hyperplane divides the space in two, a region $R_{ij}$ that contains $\bS{x}_i$, and its complement $\overline{R}_{ij}$. Then, a third node $k$ will be connected to $i$ only if $\bS{x}_k \in R_{ij}$. 
If $\bS{x}_k \in \overline{R}_{ij}$, $\theta_{ik} = 0$ and we say that $k$ has been eliminated by the hyperplane created by $j$.
The inductive application of the KRI theorem to $\bS{x}_i$ 
leads to 
a convex polytope around node $i$ disconnecting all the other points outside the polytope. 
Intuitively, NNK ignores data that are further away along a \emph{similar} direction as an already chosen point and looks for neighbors in an \emph{orthogonal} direction.

\section{Channel-wise NNK (CW-NNK) graphs}
\label{sec:cwnnk}
In intermediate representations of CNNs, very high-dimensional feature vectors can be seen as the aggregation of multiple subvectors, i.e., channels. 
We propose to construct channel-wise NNK (CW-NNK) graphs, i.e., use the data representations of the dataset in each individual channel of a specific layer to construct $C$ independent graphs.
The geometric properties of CW-NNK graphs, derived in Section \ref{ssec:cw_analysis}, allow us to analyze \emph{efficiently} the amount of overlap between channels. 
Independence between the sets of channel-wise NNK neighbors can serve as an indication that each channel specializes to different features, whereas an overlap between the sets corresponds to redundancy in the features obtained.
Furthermore, studying the size and the similarity between CW-NNK neighborhoods we can develop a better understanding of why the ID can be relatively low even in a high dimensional space. In this context, the average number of NNK neighbors can be viewed  as a characteristic parameter of the data, and 
following the ideas in \cite{campadelli2015intrinsic}, we could say that a smaller number of neighbors (more overlap between channels) is indicative of lower ID. 

\subsection{CW-NNK graphs analysis}
\label{ssec:cw_analysis}


A theoretical analysis that builds on the graphs obtained from the NNK optimization in individual channels allows us to infer relevant properties of the graph that we would obtain in the aggregate high dimensional feature space of multiple channels. 
For simplicity, we consider a scenario with two channels and their aggregate:
\[\bS{x}_i = 
    \begin{bmatrix}
        \bS{x}_{i_1} \\
        \bS{x}_{i_2} \\
    \end{bmatrix} 
    \in \mathbb{R}^D\]
where $\bS{x}_{i_1} \in \mathbb{R}^{D_1}$ and $\bS{x}_{i_2} \in \mathbb{R}^{D_2}$ are the two channels forming $\bS{x}_i$ and $D_1 + D_2=D$, but all the results presented in this section can be extended to the multiple channel case.
We analyze the properties of the set of NNK neighbors $\mathcal{N}(\bS{x}_i)$ in the aggregate space given the neighbor sets of each channel, $\mathcal{N}(\bS{x}_{i_1})$ and $\mathcal{N}(\bS{x}_{i_2})$.  
An extended analysis of this method and its applications is carried out in \cite{xBonet21}.

\begin{theorem}
\label{main_theorem}

If $j \in \mathcal{N}(\bS{x}_{i_1}) \cap  \mathcal{N}(\bS{x}_{i_2})$, and $j \in \mathcal{S}(\bS{x}_i)$, then $j \in \mathcal{N}(\bS{x}_i)$.
\end{theorem}

\begin{proof}
Consider a three-node scenario with $j$, $k$, and query $i$. We know that $\theta_{i_1,j_1} > 0$ and $\theta_{i_2,j_2} > 0$. Based on the KRI theorem \eqref{eq:kri}:
\begin{equation}\label{eq:sub1}
    \theta_{i_1,j_1} > 0 \Longleftrightarrow \bS{K}_{j_1,k_1} < \frac{\bS{K}_{i_1,j_1}}{\bS{K}_{i_1,k_1}}
\end{equation}

\begin{equation}\label{eq:sub2}
    \theta_{i_2,j_2} > 0 \Longleftrightarrow \bS{K}_{j_2,k_2} < \frac{\bS{K}_{i_2,j_2}}{\bS{K}_{i_2,k_2}}
\end{equation}

This implies that $j$ is not eliminated by any other hyperplane created by a third point $k$ in any of the two channels (and thus it is an NNK neighbor in both channels). 

Then, in the aggregate space we have $\bS{K}_{i,j} = \bS{K}_{i_1,j_1} \bS{K}_{i_2,j_2}$ since $\bS{K}_{i,j} = k(\bS{x}_i, \bS{x}_j) = \text{exp} \left( -\sum_{c=1}^{C} \|\bS{x}_{i_s} - \bS{x}_{j_s}\|^2/2\sigma^2 \right)$ and therefore $j \in \mathcal{N}(\bS{x}_i)$ if 

\begin{equation}\label{eq:agg}
\theta_{i,j} > 0 \Longleftrightarrow \bS{K}_{j_1,k_1}\bS{K}_{j_2,k_2} < \frac{\bS{K}_{i_1,j_1}\bS{K}_{i_2,j_2}}{\bS{K}_{i_1,k_1}\bS{K}_{i_2,k_2}}. 
\end{equation}
Considering \eqref{eq:sub1} and \eqref{eq:sub2}, let $\bS{K}_{j_1,k_1} = a$, $\bS{K}_{j_2,k_2} = b$, $\frac{\bS{K}_{i_1,j_1}}{\bS{K}_{i_1,k_1}} = a + \gamma$ and $\frac{\bS{K}_{i_2,j_2}}{\bS{K}_{i_2,k_2}} = b + \epsilon$. Then, we can substitute terms in \eqref{eq:agg}:
$$ab < (a+\gamma)(b+\epsilon),\; \; \text{so that}$$
$$a\epsilon+b\gamma+\gamma\epsilon > 0 \Longleftrightarrow \theta_{i,j} > 0,$$ 
$$\text{where } 0 \leq a,b \leq 1\text{ and }\gamma,\epsilon > 0\text{. Therefore, }\theta_{i,j} > 0\text{ and }$$
{\begin{equation}
    \boxed{j \in \mathcal{N}(\bS{x}_i) \Longleftrightarrow j \in \mathcal{N}(\bS{x}_{i_1}), j \in \mathcal{N}(\bS{x}_{i_2}), j \in \mathcal{S}(\bS{x}_i) }
\end{equation}}

There are no extra conditions for the initial sets of neighbors $\mathcal{S}(\bS{x}_{i_1})$ and $\mathcal{S}(\bS{x}_{i_2})$. First, we know that $j \in \mathcal{S}(\bS{x}_{i_1}) \cap \mathcal{S}(\bS{x}_{i_2})$. However, if there exists a third point $k \in \mathcal{S}(\bS{x}_i)$ that has not been selected in the initialization of one of the channels, e.g., $k \in \mathcal{S}(\bS{x}_{i_1})$ and $k \notin \mathcal{S}(\bS{x}_{i_2})$, we also want to verify \eqref{eq:agg} for this third point, but $\bS{K}_{j_2,k_2}$ is unknown. We do know, however, that $0 \leq \bS{K}_{i,j} \leq 1$, and $\bS{K}_{i_2,k_2} < \bS{K}_{i_2,j_2}$ because $j \in \mathcal{S}(\bS{x}_{i_2})$ and $k \notin S_2$. Then, $\bS{K}_{j_2,k_2} < \frac{\bS{K}_{i_2,j_2}}{\bS{K}_{i_2,k_2}}$ so \eqref{eq:sub2} is fulfilled and \eqref{eq:sub1} is also fulfilled since $k \in \mathcal{S}(\bS{x}_{i_1})$, therefore \eqref{eq:agg} is also fulfilled. Also, condition $j \in \mathcal{S}(\bS{x}_i)$ can be easily met by selecting $\mathcal{S}(\bS{x}_i) = \mathcal{S}(\bS{x}_{i_1}) \cup \mathcal{S}(\bS{x}_{i_2})$. \qedhere
\end{proof}


\begin{corollary} \label{main_corollary}
$\mathcal{N}(\bS{x}_{i_1}) \cap \mathcal{N}(\bS{x}_{i_2}) \subseteq \mathcal{N}(\bS{x}_i)$ if $\mathcal{N}(\bS{x}_{i_1}) \cap \mathcal{N}(\bS{x}_{i_2})  \subseteq \mathcal{S}(\bS{x}_i)$.
\end{corollary}
Any point $j$ such that $j \in \mathcal{N}(\bS{x}_{i_1}) \cap \mathcal{N}(\bS{x}_{i_2})$ verifies both \eqref{eq:sub1} and  \eqref{eq:sub2}. Thus, \eqref{eq:agg} is verified for $j$ as well.
By constructing only $C$ graphs, i.e., one per channel, Theorem \ref{main_theorem} and Corollary \ref{main_corollary} allow us to determine, from the intersections of neighborhoods in different channels ($j \in \mathcal{N}(\bS{x}_{i_1}) \cap \mathcal{N}(\bS{x}_{i_2})$) properties of the neighborhood in the aggregate ($\mathcal{N}(\bS{x}_i)$)  
\textit{without having to construct the aggregate NNK graph, which involves additional complexity}. Thus, $C$ graph constructions provide a lower bound on the overall overlap, without the need to construct $\sum_{n=2}^{C} \binom{C}{n}$ graphs to find the neighbors in all possible aggregations of channels, which would be impractical. 

\begin{theorem}
In a three-node scenario, if $j \in \mathcal{N}(\bS{x}_{i_1}) \cap \in \mathcal{N}(\bS{x}_{i_2})$ and we have that $k \notin \mathcal{N}(\bS{x}_{i_1})$ and $k \notin \mathcal{N}(\bS{x}_{i_2})$, then $j$ eliminates $k$ in both channels and $k \notin \mathcal{N}(\bS{x}_i)$.
\end{theorem}
\begin{proof}
If $\theta_{i_1,k_1} = 0$ because of the hyperplane created by $\theta_{i_1,j_1} > 0$, and $\theta_{i_2,k_2} = 0$ because of the hyperplane created by $\theta_{i_2,j_2} > 0$, from the KRI theorem \eqref{eq:kri}:
\begin{equation}\label{eq:zsub1}
    \theta_{i_1,k_1} = 0 \Longleftrightarrow
    \frac{1}{\bS{K}_{j_1,k_1}} < \frac{\bS{K}_{i_1,j_1}}{\bS{K}_{i_1,k_1}}
\end{equation}
\begin{equation}\label{eq:zsub2}
    \theta_{i_2,k_2} = 0 \Longleftrightarrow
    \frac{1}{\bS{K}_{j_2,k_2}} < \frac{\bS{K}_{i_2,j_2}}{\bS{K}_{i_2,k_2}}
\end{equation}
Then, in the aggregate, $k \in \mathcal{N}(\bS{x}_i)$ if
\begin{equation}\label{eq:zagg}
\theta_{i,k} > 0 \Longleftrightarrow 
\frac{1}{\bS{K}_{j_1,k_1}\bS{K}_{j_2,k_2}} >
\frac{\bS{K}_{i_1,j_1}\bS{K}_{i_2,j_2}}{\bS{K}_{i_1,k_1}\bS{K}_{i_2,k_2}}.
\end{equation}
Considering \eqref{eq:zsub1} and \eqref{eq:zsub2}, let $\frac{1}{\bS{K}_{j_1,k_1}} = a$, $\frac{1}{\bS{K}_{j_2,k_2}} = b$, $\frac{\bS{K}_{i_1,j_1}}{\bS{K}_{i_1,k_1}} = a + \gamma$ and $\frac{\bS{K}_{i_2,j_2}}{\bS{K}_{i_2,k_2}} = b + \epsilon$:
\vspace{-10pt}
$$ab \stackrel{?}{>} (a+\gamma)(b+\epsilon)$$
$$a\epsilon+b\gamma+\gamma\epsilon \nless 0$$
$$\text{where } 0 \leq a,b \leq 1\text{ and }\gamma,\epsilon > 0\text{. Thus, }\theta_{i,k} = 0\text{ and }k \notin \mathcal{N}(\bS{x}_i).$$
In words, if a point $k$ is eliminated by a hyperplane created by the same point $j$ in all subvectors, $k$ will not be connected as an NNK neighbor in the aggregate space.
\qedhere
\end{proof}

\begin{lemma}
If $k \notin \mathcal{N}(\bS{x}_{i_1})$ and $k \notin \mathcal{N}(\bS{x}_{i_2})$, but in a three-node scenario with $i$ and $j$, $j$ only eliminates $k$ in one channel,  $\theta_{i_1,k_1} = 0$ and $\theta_{i_2,k_2} > 0$, then it is possible that $k \in \mathcal{N}(\bS{x}_i)$.
\end{lemma}
\begin{proof}
Consider node $k$ which is not selected as an NNK neighbor in any of the two channels. In a three-node scenario with points $i$ and $j$ in channel 1, $\theta_{i_1,k_1} = 0$ because of the hyperplane created by $\theta_{i_1,j_1}$. But in a three-node scenario with the same points in channel 2, $\theta_{i_2,k_2} > 0$, so $k$ is eliminated from $\mathcal{N}(\bS{x}_{i_2})$, not because of the hyperplane created by $\theta_{i_2,j_2}$, but by the hyperplane created by a fourth point $q$, denoted by its normal $\theta_{i_2,q_2} > 0$. 

Then, consider \eqref{eq:zagg} and let $\frac{1}{\bS{K}_{j_1,k_1}} = a$, $\frac{1}{\bS{K}_{j_2,k_2}} = b+\epsilon$, $\frac{\bS{K}_{i_1,j_1}}{\bS{K}_{i_1,k_1}} = a + \gamma$ and $\frac{\bS{K}_{i_2,j_2}}{\bS{K}_{i_2,k_2}} = b$. Then,
\vspace{-7pt}
$$a(b+\epsilon) \stackrel{?}{>} (a+\gamma)b$$
\vspace{-13pt}
$$a\epsilon \stackrel{?}{>} b\gamma$$

where $0 \leq a,b \leq 1$ and $\gamma,\epsilon > 0$ considering \eqref{eq:zsub1} and \eqref{eq:zsub2}.
Therefore, a node $k$ that is not an NNK neighbor of $i$ in any channel, but has not been eliminated by the same node $j$ in all channels can still be selected as NNK neighbor in the aggregate space. \qedhere
\end{proof}

\section{Experiments}
\label{sec:exp}

We use CW-NNK graphs to study the overlap between channels within a CNN layer and how this overlap affects the dimensionality of the layer.
We consider a 7 layer CNN model composed of 6 convolutional layers of 16 depth channels with ReLU activations, interleaving max-pooling every 2 layers and 1 fully connected softmax layer. A dropout layer \cite{srivastava2014dropout} is added after the ReLU activations for regularization.
Models are trained on the well-known CIFAR-10 dataset \cite{krizhevsky2009learning} for 200 epochs with the Adam optimizer with a learning rate of 0.001, and a batch size of 50.
A class-balanced subset of 1000 randomly sampled data points from the original train set is used to construct CW-NNK graphs. Given a query data point, the rest of the subset is used as train data for graph construction. The intermediate representations of the input data (feature maps) at the output of a convolutional block are used as feature vectors.

\subsection{Channel redundancy across CNN layers}
\label{ssec:redundancy_layers}


Section \ref{ssec:cw_analysis} developed properties showing that when an NNK neighbor is common across two channels it will also be an NNK neighbor in the aggregation of those two channels, without having to construct the graph for the aggregation. 
We next show empirically that these overlaps of neighborhoods in two channels for a given instance do happen in practice. 
%
For each convolutional layer, we compute the ratio between the number of pairwise channel NNK intersections and the the average number of NNK neighbors per channel (CW-NNK overlap), and we average it for all the data points in the subset. \autoref{fig:intersections_across_layers} shows that there is indeed a significant number of neighbors overlapping in different channels in the shallower layers. The CW-NNK overlap is reduced in the deeper layers, reaching a minimal point in the last convolutional layer, denoting more independence between channels. This confirms what has been observed in other works that study channel redundancy by different means \cite{shang2016understanding, he2017channel}.

\begin{figure}[!ht]
    \centering
    \includegraphics[width=\linewidth]{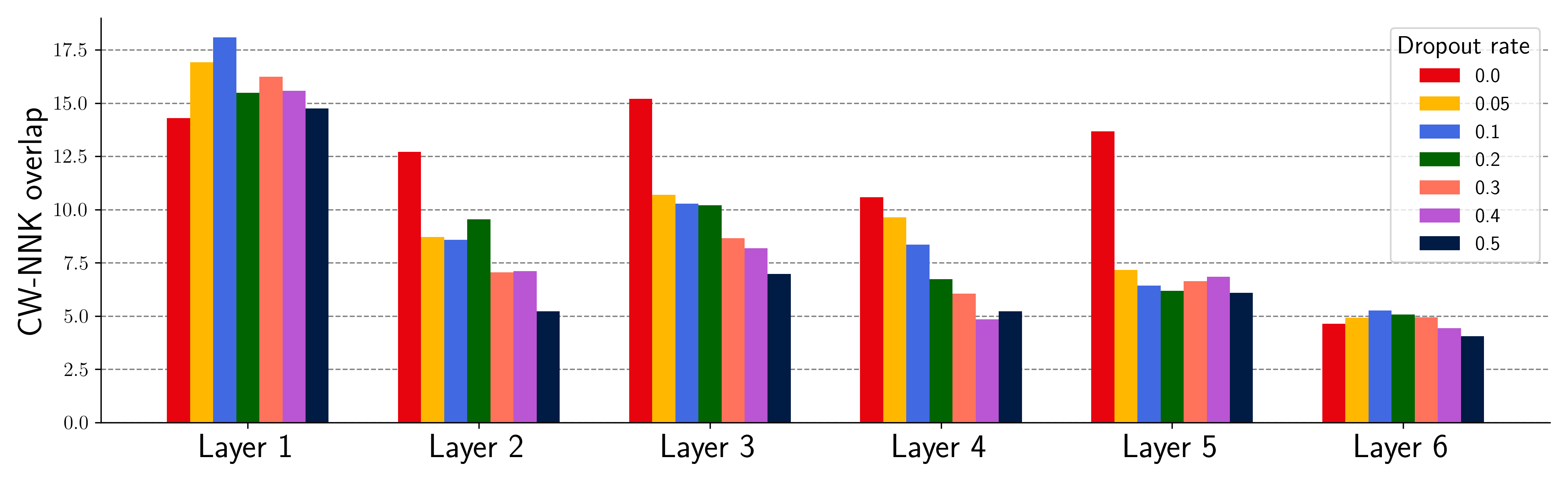}
    \vspace{-15pt}
    \caption{Number of pairwise channel NNK intersections to average number of NNK neighbors per channel ratio (CW-NNK overlap) in each layer for different dropout rates. 
    }
    \vspace{-13pt}
    \label{fig:intersections_across_layers}
\end{figure}

We also observe that channel redundancy in the intermediate layers (second to penultimate layer) depends on the level of model regularization. Higher dropout rates lead to channels of intermediate layers having less overlap, while lack of  regularization results in higher overlap. 
These observations are promising for further improvement of channel pruning techniques and redundancy constrains to improve training.
However, less overlap between channels does not necessarily imply higher generalization. As shown in \cite{kahatapitiya2021exploiting}, over-specialized channels can lead to worse performance \cite{kahatapitiya2021exploiting}, and a controlled level of redundancy is preferable to maintain performance. We discuss the CW-NNK overlap in the last convolution layer and its relation to ID and generalization next. 

\subsection{CW-NNK overlap and ID as generalization indicators}

\label{ssec:overlap_generalization}
\begin{figure}[hb]
\centering
\begin{subfigure}[b]{.48\linewidth}
  \centering
  \includegraphics[width=\linewidth]{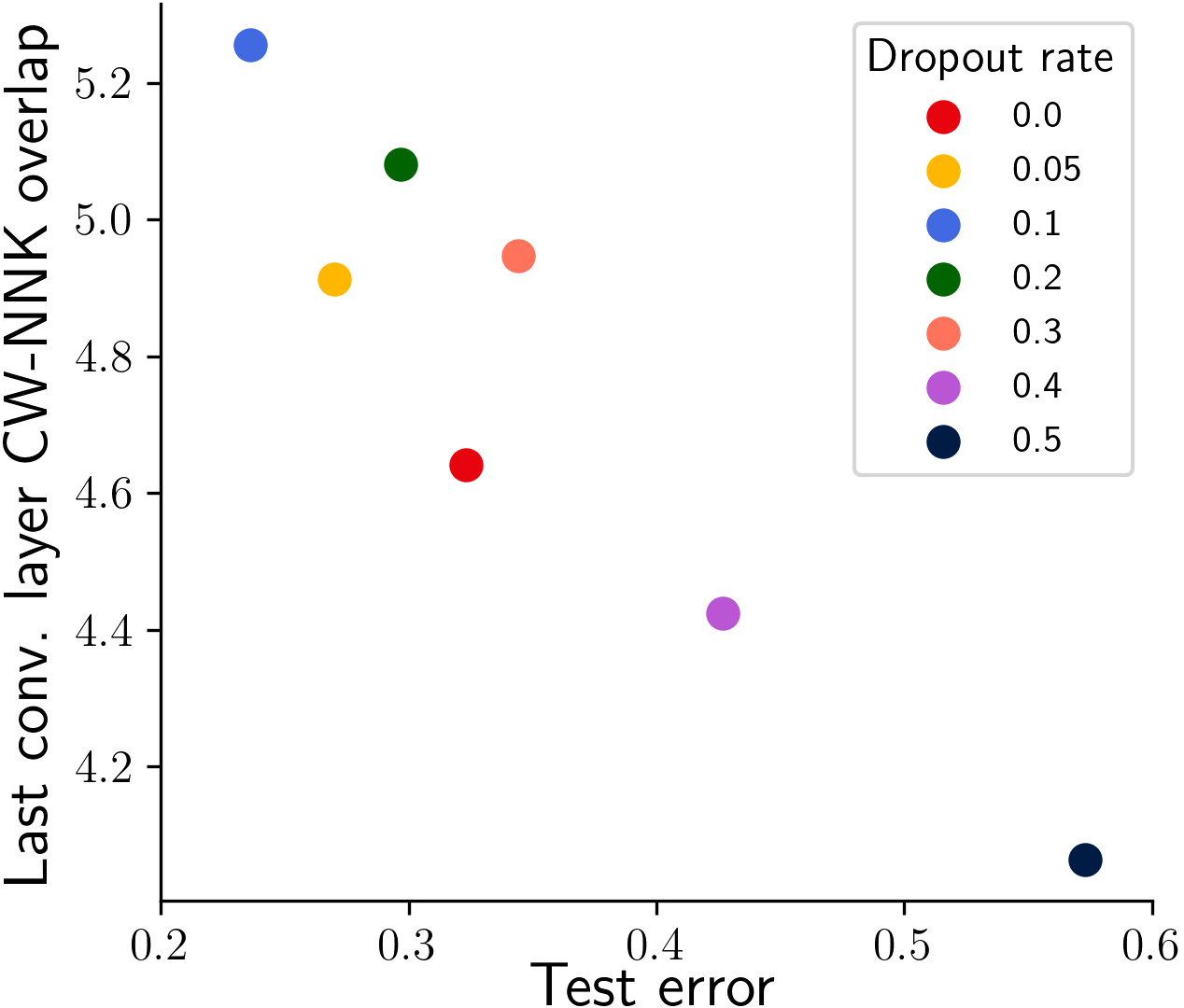}
  \vspace{-10pt}
  \caption{}
  \label{fig:test_error_overlap}
\end{subfigure}
\begin{subfigure}[b]{0.48\linewidth}
  \centering
  \includegraphics[width=0.8\linewidth]{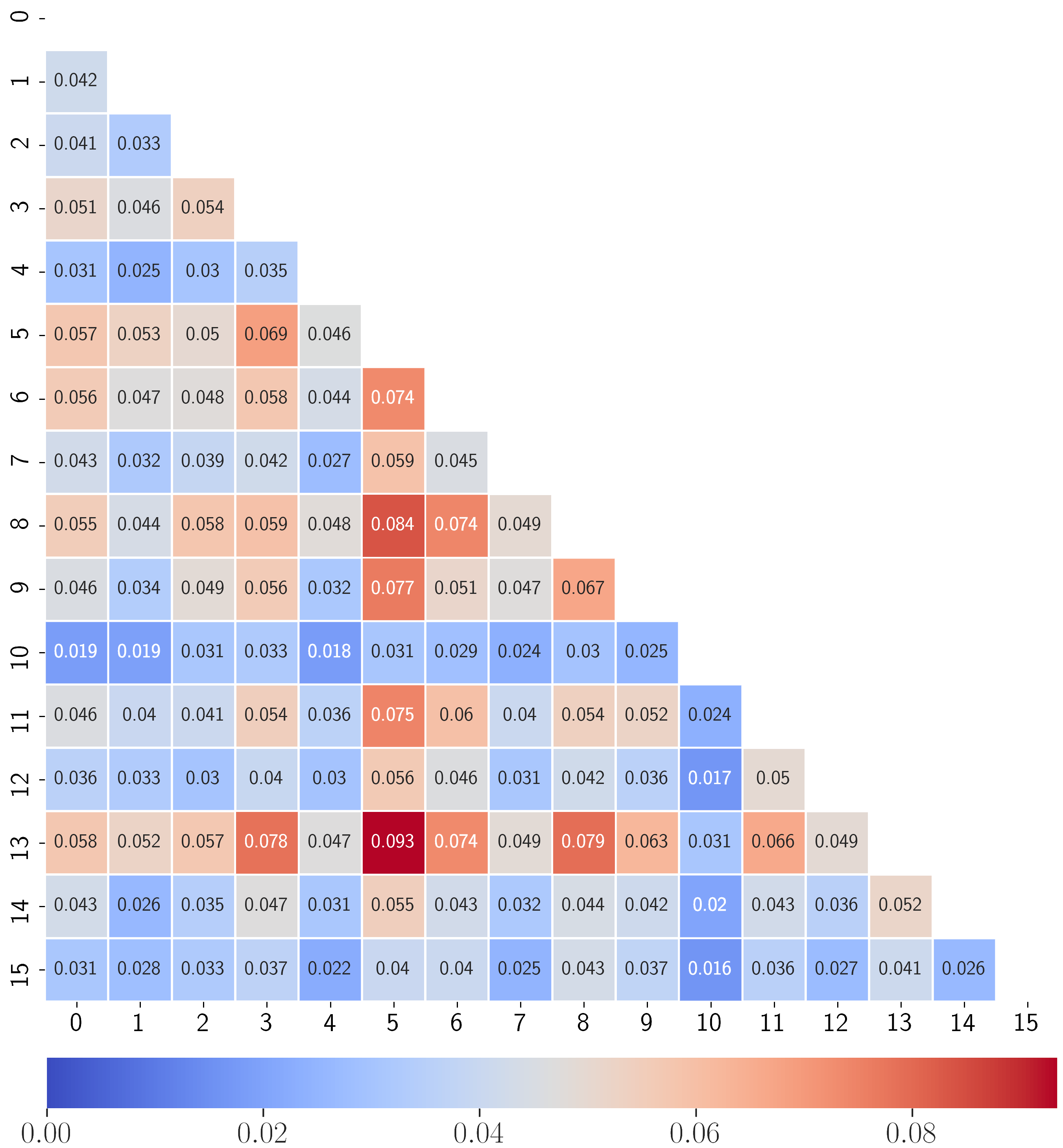}
  \caption{}
  \label{fig:intersections_matrix}
\end{subfigure}
\vspace{-8pt}
\caption{(a) Last convolutional layer CW-NNK overlap as a function of model test error. (b) CW-NNK overlap for all pairs of channels in the last convolutional layer.}
\label{fig:last_layer_overlap}
\vspace{-3pt}
\end{figure}

As observed in \cite{ansuini2019intrinsic}, the ID in the last layer on the training set is a strong predictor for generalization performance. Between multiple models, lower IDs strongly correlate with higher classification accuracy. 
In general, the information overlap between channels plays a key role in the estimation of the overall ID of the data manifold. The ID of an aggregate of multiple channels may be additive if there is no overlap between them, while if high overlap is observed, the ID of the aggregate could be similar to the ID of a single channel, related to the average number of CW-NNK neighbors. 
We compare the CW-NNK overlap in the last convolutional layer
with the test error of our model using different dropout rates. Since the models are trained for 200 epochs without early stopping, lack of regularization leads to overfitting to the training data, i.e., high test error. However, an excessive regularization, i.e., high dropout rates, can cause underfitting \cite{srivastava2014dropout}, leading to both train and test errors to be high. \autoref{fig:test_error_overlap} shows how the CW-NNK overlap strongly correlates with test error, demonstrating that a high CW-NNK overlap in the last convolutional layer is indicative of a lower ID of the aggregate feature vector, and thus, of higher generalization performance. 
The last convolutional layer results do not contradict the observations in Section \ref{ssec:redundancy_layers}, since less redundancy in intermediate layers, owing to regularization, does not necessarily affect generalization performance.

\vspace{-2pt}
\subsection{CW-NNK overlap distribution and examples}

Given that there is a considerable overlap of neighbors between different channels in a layer, we want to observe whether these intersections are randomly distributed among all channels or follow a pattern.
\autoref{fig:intersections_matrix} represents the CW-NNK overlap for all pairs of channels in the last convolutional layer of the trained model that achieves highest test accuracy, with a dropout rate of 0.1. It is important to note that with only $C$ graphs we can observe similarities beyond pairs of channels (i.e., neighbors that intersect in trios, quartets, etc., of channels) and in this matrix we represent all pairwise combinations of channels that occur.
We observe that, in both layers, most of the intersections are concentrated in a few channels, which capture similar features. On the other hand, other channels are very specialized and independent, and since the CW-NNK overlap is very low, the captured features differ from those of other channels. 

A major advantage of this method is that the feature overlap can be interpreted directly from the input images. \autoref{fig:neighbors} shows the sets of neighbors of three channels of the last convolutional layer for a given query. We can observe how three images appear in the neighborhood of the first two channels. All the other neighbors selected in these channels share similar general features (vehicles, light or blue backgrounds, etc.). However, the bottom channel does not overlap with the other two, and we can confirm that the channel captures very different features from the input data, selecting data points from other non-vehicle classes as neighbors with very different features.

\begin{figure}[!ht]
    \centering
    \includegraphics[width=0.95\linewidth]{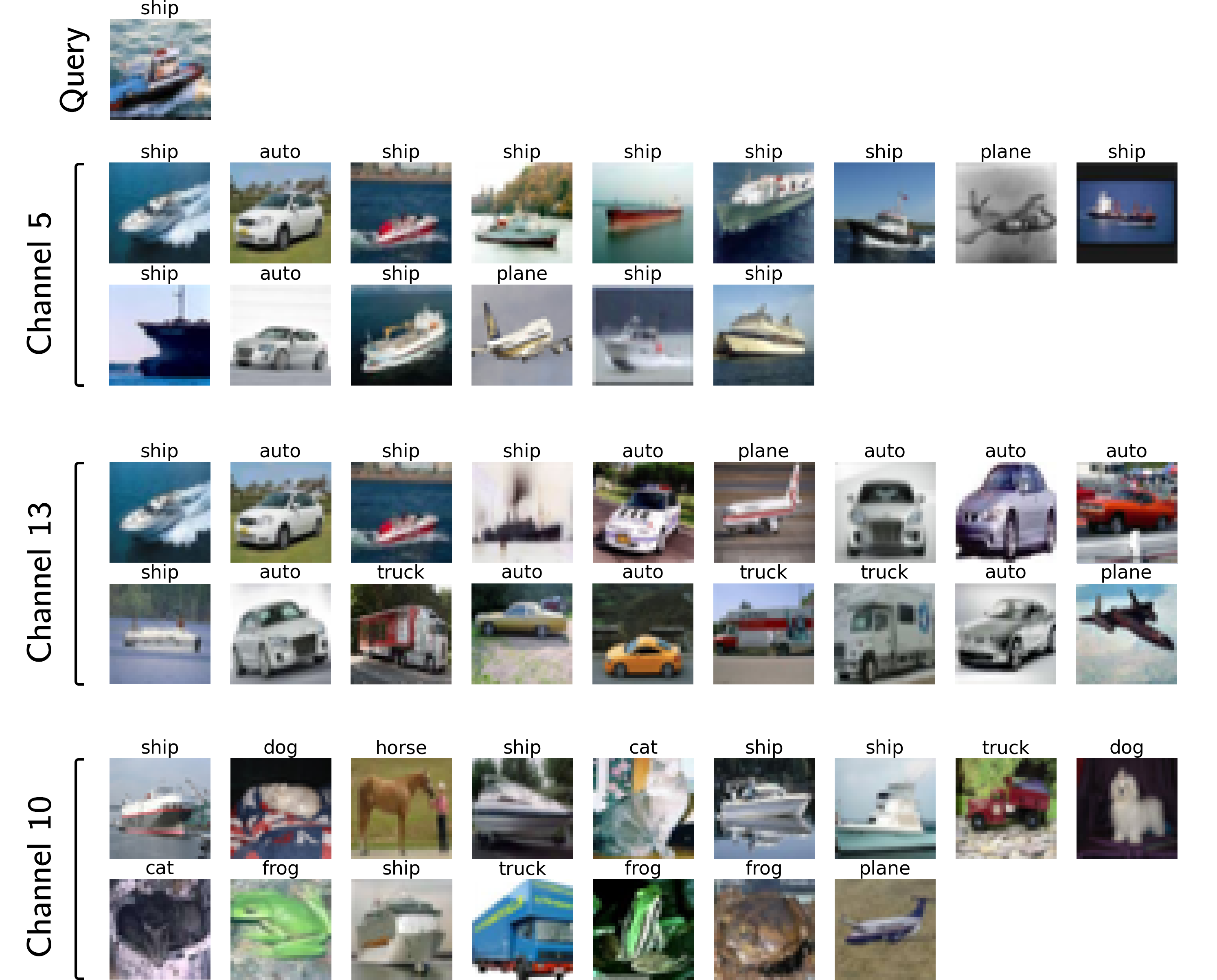}
    \vspace{-2pt}
    \caption{Sets of neighbors for three channels of the last convolutional layer for a given query. Channels 5 and 13 overlap and share similar features while channel 10 captures very different features.}
    \label{fig:neighbors}
\end{figure}


\vspace{-10pt}
\section{Conclusion}
\label{sec:conclusion}

We study the conditions that allow computing the overlap between channels efficiently through CW-NNK graphs, and which effect it has on the overall ID of the data. We discuss the impact of regularization and layer depth on channel redundancy.
Finally, we observe that CW-NNK overlap in the last convolutional layer is strongly correlated with generalization performance.
Future work will consider using the level of overlap for channel pruning and to improve generalization estimation for channel-wise early stopping \cite{bonet2021channel}.

\newpage
\vfill\pagebreak

\bibliographystyle{myIEEEbib}
\bibliography{refs}

\end{document}